\documentclass[letterpaper, 10 pt, conference]{ieeeconf}

\IEEEoverridecommandlockouts                            
\usepackage{amsthm}
\usepackage{amsmath}
\usepackage{amssymb}

\theoremstyle{definition}
\newtheorem{definition}{Definition}%[section]
\newtheorem{proposition}{Proposition}
\newtheorem{corollary}{Corollary}

\usepackage[english]{babel}
\usepackage{pgf}
\usepackage[utf8x]{inputenc}
\usepackage{caption}  
\DeclareCaptionFont{mysize}{\fontsize{8}{9.6}\selectfont}
\usepackage{cite}
\usepackage{flushend}
\usepackage{graphicx}
\usepackage{listings}
\usepackage{courier}

\usepackage{algpseudocode}
\usepackage{algorithm}
\usepackage{url}

\usepackage{mwe}
\usepackage{relsize}
\let\labelindent\relax
\usepackage{enumitem}

\usepackage{booktabs}

\newcommand{\Tau}{\mathrm{T}}
\newcommand{\indep}{\perp \!\!\! \perp}

\title{\LARGE \bf
Autonomous Task Planning for Heterogeneous Multi-Agent Systems*
}

\author{Anatoli A. Tziola and Savvas G. Loizou
\thanks{*This work was partially supported by European Union's Horizon 2020 research and innovation program under grant agreements 767642 (L4MS) and 951813 (Better Factory).}
\thanks{The authors are with the Department of Mechanical Engineering and Materials Science and Engineering,
		Cyprus University of Technology, Limassol, Cyprus,
        {\tt\small \{anatoli.tziola, savvas.loizou\}@cut.ac.cy}.}
}

\begin{document}

\maketitle

\thispagestyle{empty}
\pagestyle{empty}

%%%%%%%%%%%%%%%%%%%%%%%%%%%%%%%%%%%%%%%%%%%%%%%%%%%%%%%%%%%%%%%%%%%%%%%%%%%%%%%%
\begin{abstract}
This paper presents a solution to the automatic task planning problem for multi-agent systems. A formal framework is developed based on the Nondeterministic Finite Automata with $\epsilon$-transitions, where given the capabilities, constraints and failure modes of the agents involved, an initial state of the system and a task specification, an optimal solution is generated that satisfies the system constraints and the task specification. The resulting  solution is guaranteed to be complete and optimal; moreover a heuristic solution that offers significant reduction of the computational requirements while relaxing the completeness and optimality requirements is proposed. The constructed system model is independent from the initial condition and the task specification, alleviating the need to repeat the costly pre-processing cycle for solving other scenarios, while allowing the incorporation of failure modes on-the-fly. Two case studies are provided: a simple one to showcase the concepts of the proposed methodology and a more elaborate one to demonstrate the effectiveness and validity of the  methodology.
\end{abstract}

%%%%%%%%%%%%%%%%%%%%%%%%%%%%%%%%%%%%%%%%%%%%%%%%%%%%%%%%%%%%%%%%%%%%%%%%%%%%%%%%
\section{INTRODUCTION}
Multi-agent systems  operating autonomously in dynamical environments to perform complicated tasks have been one of the major areas of research interest during the last decade. High-level task planning using formal methods to  define the system requirements is one of the promising approaches. Typical objectives arise from the multi-agent systems' behavior and requirements, such as sequential or reactive tasks, control, coordination and motion and task planning.

Several proposed methodologies address the high-level task planning problem using formal languages to express autonomous systems behavior. Some of the most common approaches include Linear Temporal Logic (LTL), sampling-based approaches and domain definition languages. Many of the existing works use LTL formulas to develop bottom-up approaches \cite{kress2007s, kress2009temporal, guo2016task, smith2011optimal} where local LTL expressions are assigned to robots, or top-down approaches \cite{chen2013formal}, \cite{kloetzer2009automatic}, \cite{quottrup2004multi}, where a global task is decomposed into independent subtasks that are treated separately by each agent. 

In \cite{kress2007s}, a high-level plan is found by a discrete planner which seeks the set of system transitions that ensure the satisfaction of a logic formula. In \cite{dimagoras2014}, a formal method based on LTL has been developed to model specifications and implement centralized planning for multi-agent systems. Some related works suggest that hierarchical abstraction techniques for single-agent systems can be extended to multi-agent systems using parallel compositions  \cite{kloetzer2009automatic}, \cite{quottrup2004multi}. The fact that the transitions should be on the common events to allow parallel execution is very restrictive. In \cite{guo2016task}, temporal logic formulas are utilized to specify sub-formulas that could be executed in parallel by the agents,  without a global task definition. In \cite{schillinger2018simultaneous}, the authors tackle the multi-robot task allocation problem under constraints defined by LTL formalism to concurrently plan tasks for robot agents. In \cite{kloetzer2016multi}, a formal method based on LTL has been deployed to model multi robot  motion planning specifications. Common among the above works is the adaptation of formal verification techniques for motion planning and controller synthesis. The two main limitations of the above works are the exponential growth of the state space even for relatively simple problems and the extra computations required to express LTL formulas in B\text{\"u}chi automata. In \cite{goryca2013formal}, authors propose a formal synthesis of supervisory control software for multi-robot systems, while the scalability of the approach was improved in \cite{hill2017scaling}.

In \cite{kantaros2020stylus}, a sampling-based approach is presented using directed trees to approximate the state space and transitions of synchronous product automata. The sampling process is guided by transitions that belong to the shortest path to the accepting states. However, these algorithms provide no solution quality guarantee. On the other hand, implementing PDDL \cite{fikes1971strips, fox2003pddl2} would be time consuming for real time applications.

Our proposed approach seeks the optimal shortest path in a weighted graph and a case study focusing on the logistics flow aspects was chosen. 

The main contribution of this work are:
\begin{itemize}
\item System modeling using nondeterministic finite automata with $\epsilon$-transitions \cite{cassandras2009introduction} expressing the system's behavior and combining the agents capabilities and constraints at the individual and group level, including failure modes.
\item Determination of the optimal task plan that satisfies the task specification with-out the need to repeat the pre-processing cycle for solving other scenarios.
\item Determination of reduced complexity sub-optimal solutions to the task planing problem.
\item Incorporation of failure modes on-the-fly, after building  the global system model (i.e. without the need to repeat the costly pre-processing step) 
\end{itemize}

This work introduces a new approach the SuPErvisory Control Task plannER (\textit{SPECTER}) for high-level task planning problems with respect to agent capabilities, constraints and failure modes. The problem formulation uses the environment model composed by individual agents' capabilities and constraints, considering the individual agent's failure modes to determine the optimal task plan. This work builds on top of \cite{loizou2005automated}. The problem is posed as a special case of module composition problem (MCP)  \cite{tripakis2003automated} and then reduced to a combinatorial optimization problem of shortest dipath, which can be solved in polynomial time. The optimal module chain, combined with the formal approach of Supervisory Control Theory (SCT) \cite{ramadge1987supervisory} can then be applied to synthesize control laws and communication strategies for efficiently accomplishing a global tasks \cite{wonham1987supremal}. Agents'  navigation can be handled by a navigation methodology with performance guarantees that ensure compositionality, like the navigation transformation \cite{loizou2022mobile}. Agents' control and navigation are not under the scope of the current work. The active source code of the software developed is available at \cite{bpogithub}. 

The rest of the paper is organized as follows: Section \ref{sec:Prelims} presents the necessary preliminary notions, section \ref{sec:PROB_FORM} presents the problem formulation using the $\epsilon_0$-NFA formalism while section \ref{sec:MC_PF} exploits the $\epsilon_0$-NFA formalism to reduce the problem to a Module Composition one. Section \ref{sec:Analysis} presents the algorithms and the analysis of the methodology while section \ref{sec:Exp} presents a case study to demonstrate the operation of SPECTER. Section \ref{sec:Concl} concludes the paper. 

\section{PRELIMINARIES}
\label{sec:Prelims}
\subsection{Definitions}
In this section, we introduce the necessary formalisms. 

If $A$ and $B$ are sets, the cardinality of set $A$ is denoted as $| A |$. The union and intersection of sets are denoted as $A\cup B$ and $A\cap B$ whereas set subtraction of set $B$ from set $A$ is denoted as $A\setminus B$. We use the $\wedge$ operator to denote conjunction.

We will use the definition of  Deterministic Finite Automata (DFA) by \cite{cassandras2009introduction}.
\begin{definition}[DFA]
A Deterministic Finite Automaton (DFA) is a six-tuple, $G = (X_G, E_G, f_G, \Gamma_G, x_{0,G}, X_{m,G})$, consisting of:
\begin{itemize}
\item a non-empty finite set of states $X_G$,
\item a non-empty finite set of events $E_G$,
\item a transition function $f_G: X_G \times E_G \rightarrow X_G$,
\item an active event function $\Gamma_G: X_G \rightarrow 2^{E_G}$, 
\item an initial state $x_{0,G} \in X_G$,
\item a finite set of marked states $X_{m,G} \subseteq X_G$
\end{itemize}
where $f_G$ is partial on its domain in the sense that if event $e \in \Gamma(x)$, then $e$ labels a transition from $x$ to a unique state $y=f_G(x,e)$.
\label{def:DFA}
\end{definition}

\begin{definition}[$\epsilon_0$-NFA]
An $\epsilon_0$-NFA is a Nondeterministic Finite Automaton with $\epsilon$-transitions only from the initial state, defined as a six-tuple, \[ {}^{\epsilon_0} G= (X_G \cup \{x_0 \}, E_G \cup \{\epsilon\}, {}^{\epsilon_0} \!f_G, \Gamma_G, x_0, X_{m,G}),\] where ${}^{\epsilon_0} \!f_G: X_G \cup \{x_0 \} \times E_G \cup \{\epsilon\} \rightarrow 2^{X_G}$ such that $\forall (x,e) \in X_G \times E_G: f_G(x,e) = {}^{\epsilon_0} \!f_G(x,e)$, where the variables with $``G"$ subscript are as in Definition \ref{def:DFA} and $x_0 \notin X_G$ is the initial state.
\label{def:NFA}
\end{definition}

\begin{corollary}
\label{cor:NFA2DFA}
If $ {}^{\epsilon_0} G$ is an $\epsilon_0$-NFA then it can be converted to the DFA $G$ by removing $x_0$ along with the associated $\epsilon$ transitions and assigning an $x_{0,G} \in X_G$ as an initial state. 
\end{corollary}

\begin{proof}
This can be trivially shown by observing that the six-tuple obtained by the operation is as the one in Definition \ref{def:DFA}.
\end{proof}

Based on the above Corollary, we can construct the following operator:

\begin{definition}[]
Let ${}^{\epsilon_0} G$ be an $\epsilon_0$-NFA and $x_{0,G}\in X_G $ a state in $X_G$ that we want to assign as an initial state. Then \[\Delta({}^{\epsilon_0} G, x_{0,G}) \triangleq G \]
where $G$ is as in Definition \ref{def:DFA} and the operator $\Delta$ performs the conversion as described in Corollary \ref{cor:NFA2DFA}.
\label{def:deltaNFA2DFA}
\end{definition}

We need to associate a cost function with the transitions:
\begin{definition}[Transition Cost Function] A transition cost function for an event set $E_G$, is defined as $g_G: E_G \rightarrow  \mathbb{R}_{>0}$.
\end{definition}

We define the inverse transition function that implies the transition relation backward derived from the transition function.
\begin{definition}[Inverse Transition Function]\label{def:inverseTF}
Let $G$ be a DFA. Define $f^{-1}_G: X_G \times E_G \rightarrow X_G$ to be the inverse transition function such that $f_G(x_G,e)=y_G \wedge f^{-1}_G(y_G,e)=x_G$ for some $x_G: e \in \Gamma_G(x_G)$, $y_G: e \in \Gamma_G(f^{-1}_G(y_G,e))$.
\end{definition}

For the automata operations that will be required in the sequel, we need to introduce the following concept:
\begin{definition}[Compatible $\epsilon_0$-NFAs]\label{def:compatibleDFAs}
Let ${}^{\epsilon_0} G$ and ${}^{\epsilon_0} \!B$ be $\epsilon_0$-NFAs.  Let $E_C=E_G \cap E_B$. ${}^{\epsilon_0} G$ and ${}^{\epsilon_0} \!B$ are compatible iff $\forall e\in E_C$ then\footnote{Note that due to Corollary \ref{cor:NFA2DFA} we have that ${}^{\epsilon_0} \!f_G (x_G,e) \equiv f_G (x_G,e)$ and ${}^{\epsilon_0} \!f_B (x_B,e) \equiv f_B (x_B,e)$.}, $f_G (x_G,e) = f_B (x_B,e)$ and  $f^{-1}_G(y_G,e) = f^{-1}_B(y_B,e)$ for some $x_G : e\in \Gamma_G(x_G)$,  $x_B :  e\in \Gamma_B(x_B)$,  $y_G : e\in \Gamma_G(f^{-1}_G(y_G,e))$ and $y_B : e\in \Gamma_B(f^{-1}_B(y_B,e))$. In such case it will also be true that $x_G=x_B$ and $y_G=y_B$. We denote such a compatibility relation as ${}^{\epsilon_0} G \asymp {}^{\epsilon_0} B$.
\end{definition}
Note that the above definition effectively forces the endpoints of common events to be common states.

\subsection{Operations on compatible automata}
Here we introduce a custom set of the basic operations on compatible $\epsilon_0$-NFAs: union, subtraction and concatenation. The introduced operations differ from the ones found in the automata literature \cite{cassandras2009introduction, hopcroft2006automata, khoussainov2012automata} and provide the  necessary functionality for the   subsequent developments.

\begin{definition}[Union of Compatible Automata]\label{def:union}
For the $\epsilon_0$-NFAs ${}^{\epsilon_0} G$ and ${}^{\epsilon_0} \!B$, assume ${}^{\epsilon_0} G \asymp {}^{\epsilon_0} \!B$. Define the compatible automata union ${}^{\epsilon_0} \mathcal{P} \triangleq {}^{\epsilon_0} G \cup_{\asymp} {}^{\epsilon_0} \!B$ to be the six-tuple:
\[{}^{\epsilon_0} \mathcal{P} = (X_{\mathcal{P}} \cup \{x_0\}, E_\mathcal{P} \cup \{\epsilon\}, {}^{\epsilon_0} \!f_\mathcal{P}, \Gamma_\mathcal{P}, x_0, X_{m,\mathcal{P}} )\] 
where $X_\mathcal{P} = X_G \cup X_B $, $E_\mathcal{P} = E_G \cup E_B$, ${}^{\epsilon_0} \!f_\mathcal{P}: X_\mathcal{P} \cup \{x_0 \} \times E_\mathcal{P} \cup \{\epsilon\} \rightarrow 2^{X_\mathcal{P}}$ such that $\forall (x,e) \in X_G \times E_G: {}^{\epsilon_0} \!f_\mathcal{P}(x,e) = f_G(x, e)$ and $\forall (x,e) \in X_B \times E_B: {}^{\epsilon_0} \!f_\mathcal{P}(x,e) = f_B(x, e)$, $\Gamma_\mathcal{P}: X_{\mathcal{P}} \rightarrow 2^{E_\mathcal{P}}$, $x_0 \notin X_{\mathcal{P}}$ and $X_{m,\mathcal{P}} = X_{m,G} \cup X_{m,B}$.
\end{definition}

In contrast with the standard union operation on automata presented in the literature, the above operation produces the union instead of the cartesian product state space.

\begin{definition}[Subtraction of Compatible Automata]\label{def:subtraction}
For the $\epsilon_0$-NFAs ${}^{\epsilon_0} G$ and ${}^{\epsilon_0} \!B$, assume ${}^{\epsilon_0} G \asymp {}^{\epsilon_0} \!B$. Define the compatible automata subtraction ${}^{\epsilon_0} \Theta \triangleq {}^{\epsilon_0} G \setminus_{\asymp} {}^{\epsilon_0} B$ to be the six-tuple:
\[ {}^{\epsilon_0} \Theta = (X_{\Theta} \cup \{x_0\}, E_{\Theta} \cup \{\epsilon\}, {}^{\epsilon_0} \!f_{\Theta}, \Gamma_{\Theta}, x_0, X_{m,\Theta} ) \]
where $X_{\Theta}= X_G$, $E_{\Theta} = E_G \setminus E_B$, ${}^{\epsilon_0} \!f_\Theta: X_\Theta \cup \{x_0 \} \times E_\Theta \cup \{\epsilon\} \rightarrow 2^{X_\Theta}$ such that $\forall (x,e) \in X_\Theta \times E_\Theta: {}^{\epsilon_0} \!f_\Theta(x,e) = f_\Theta(x, e)$, $\Gamma_{\Theta} : X_{\Theta} \rightarrow 2^{E_{\Theta}}$, $x_0 \notin X_\Theta$, $X_{m,\Theta} = X_{m,G} \setminus X_{m,B}$.
\end{definition}

Note that this is more in line with the set difference operator than with the intersection with the language complement that is used in regular languages.

\begin{definition}[Concatenation of Compatible Automata]\label{def:concatenation}
For the $\epsilon_0$-NFAs ${}^{\epsilon_0} G$ and ${}^{\epsilon_0} \!B$, assume ${}^{\epsilon_0} G \asymp {}^{\epsilon_0} \!B$. Define the compatible automata concatenation ${}^{\epsilon_0} \Phi \triangleq {}^{\epsilon_0} G \indep _\asymp {}^{\epsilon_0} \! B$ to be the six-tuple:
\[ {}^{\epsilon_0} \Phi = (X_{\Phi} \cup \{x_0\}, E_{\Phi} \cup \{\epsilon\}, {}^{\epsilon_0} \!f_{\Phi}, \Gamma_{\Phi}, x_0, X_{m,\Phi} ) \]
where $X_{\Phi}=\{ uv | u \in X_G, v \in X_B\}$, $E_{\Phi}=E_G \cup E_B$, ${}^{\epsilon_0} \!f_\Phi: X_\Phi \cup \{x_0 \} \times E_\Phi \cup \{\epsilon\} \rightarrow 2^{X_\Phi}$ such that $\forall (x,e) \in X_\Phi \times E_\Phi: {}^{\epsilon_0} \!f_\Phi(x,e) = f_\Phi(x,e)$, $\Gamma_{\Phi}: X_{\Phi} \rightarrow 2^{E_\Phi}$, $x_0 \notin X_\Phi$, $X_{m,\Phi}= \{ uv | u \in X_{m,G} \wedge v \in X_{m,B} \} $.
\end{definition}

Note that the operation of concatenation in Definition \ref{def:concatenation} above, is similar to the well known parallel composition operation (see e.g. \cite{cassandras2009introduction}) but differs in that the event sets undergo a disjoint union in our case to ensure that only one event can be activated at a time (no synchronization on common events).

The concept of compatible automata, allows us to recover the following  properties for the operations of Definitions \ref{def:union},  \ref{def:subtraction} and \ref{def:concatenation}:
\begin{corollary}
The automata resulting from the operations defined in Definitions \ref{def:union},  \ref{def:subtraction} and \ref{def:concatenation} are $\epsilon_0$-NFAs.
\label{cor:ResultNFA}
\end{corollary}
\begin{proof}
This is trivially established by noting that for the operations in Definitions \ref{def:union},  \ref{def:subtraction} and \ref{def:concatenation}, the resulting automata are six-tuples as in Definition \ref{def:NFA}. Moreover if we assume that $R$ is the resulting automaton from each operation, then due to the properties of Definition \ref{def:compatibleDFAs}, it will hold that for each operation $\forall x \in X_R$, then $\forall e \in \Gamma_{G}(x) \cap \Gamma_{B}(x)$ (if such exists)  such that $f_{G} (x,e) = f_{B} (x,e)$. Hence, $R$ will be an $\epsilon_0$-NFA since it fulfils the requirements of Definition \ref{def:NFA}.
\end{proof}

The principle of completeness asserts that the algorithm always returns a solution (if one exists), otherwise if there is no solution, the algorithm reports failure. 

We need to introduce the following operator that addresses individual states of  concatenated states of Definition \ref{def:concatenation}.
\begin{definition}[Projection Operator]\label{def:projection}
Assume $x \in X_G$ is a state of $G$ with $|x|=n$ and let $x_i$ denote the $i$'th element of $x$. Define the projector $b$ as an $n$-bit binary and $b_i$ its $i$'th bit. The projection operator is defined as $proj(x,b)\triangleq \left\{x_i|b_i=1\right\}$.
\end{definition}

We will need some definitions from MCP literature as well as some new ones for our development. Using the module definition by \cite{tripakis2003automated}, we state the module in an appropriate form for our development.
\begin{definition}[Single Port Module]\label{def:taskmodule}
Let $G$ be a DFA. Consider the finite set of ports $P=P_{in} \cup P_{out}$, where $P_{in}$ be a non-empty finite set of input ports and $P_{out}$ be a non-empty finite set of output ports with $P \subseteq X_G$ and $P_{in} \cap P_{out} \neq \emptyset$. 
We define the single input/single output port (I/O) module $\Tau$ as a 3-tuple $\Tau \triangleq \{ p, e, q \}$, where $p \in P_{in}$ is the input port, $e \in E_G: f_G(p,e)=q$ and $q \in P_{out}$ is the output port.
\end{definition}
$\Tau$ provides an output $q$ as a response to the corresponding input $p$. The cost associated with the module $\Tau$ is defined as $c(\Tau) \triangleq g_G(e)$. We can also define the inverted module $\Tau^{-1}$ where its input becomes its output and vice versa such that Definition \ref{def:inverseTF} holds. 

We consider the notion of directional compatibility of single port modules that specifies the interaction between the modules.
\begin{definition}[Directional Compatibility of I/O Modules]\label{def:directionalCompatibility}
Let $G$ be a DFA and $\Tau_\chi = \{ p_\chi, e_\chi, q_\chi \}$ and $\Tau_\psi= \{ p_\psi, e_\psi, q_\psi \}$ be modules. We define the directional compatibility relation between $\Tau_\chi$ and $\Tau_\psi$ and we write $\Tau_\chi \rightharpoonup \Tau_\psi$ iff $q_{\chi} = p_{\psi}$.
\end{definition}

\section{PROBLEM FORMULATION}
\label{sec:PROB_FORM}
\subsection{Agent Model}
Agents are considered as autonomous entities, such as humans, robots, items, machines or anything that it could change its status to act on or react to its surroundings during a process or trigger event. An agents is modeled as an $\epsilon_0$-NFA composed by the agent's capabilities and constraints. Agent's capabilities represent the allowed state transitions derived from the combination of individual capabilities and failure mode, expressed as $\epsilon_0$-NFAs. Agent's constraints express the forbidden state transitions derived from the union of individual's constraints modeled as $\epsilon_0$-NFAs. The agent model is produced by the subtraction of the agent's constraints $\epsilon_0$-NFA from the agent's capabilities $\epsilon_0$-NFA.

Let $n$ be the total number of agents and ${}^{\epsilon_0} \!\mathcal{A}$ be the finite set of the agents $\epsilon_0$-NFAs, where $n = |{}^{\epsilon_0} \!\mathcal{A}|$. We represent the $i^{th}$ agent as ${}^{\epsilon_0} \!A_i$, $i \in \{1, \ldots n\}$. 
\subsubsection{Individual Agent Capabilities}
Let $\kappa$ be the total number of individual capabilities of ${}^{\epsilon_0} \!A_i$ and ${}^{\epsilon_0} \!M_{\beta,i}$ the $\beta$'th individual capability of ${}^{\epsilon_0} \!A_i$, where $\beta \in \{1, \ldots \kappa\}$.

\subsubsection{Individual Agent Failure Mode}
Consider a state $q' \in X_{M_{\beta,i}}$. We define a failure mode of ${}^{\epsilon_0} \!A_i$ as the inability to complete the transition from some $q\in X_{M_{\beta,i}}$ to $q'$ with an occurrence of event $e \in E_{M_{\beta,i}}$. This describes a detected transition failure of ${}^{\epsilon_0} \!A_i$ which renders  $f_{M_{\beta,i}}(q, e)=q'$ infeasible. This failure mode is modeled by the $\epsilon_0$-NFA ${}^{\epsilon_0} \!F_i$ such that $X_{F_i}=\{q, q'\}$ and $E_{F_i}=\{e\}$.

We model the agent's capabilities as the subtraction of the agent failure from the union of individual agent capabilities utilizing the union and the subtraction operations of compatible automata.
\subsubsection{Agent Capabilities}
Considering $\kappa$ compatible $\epsilon_0$-NFAs such that ${}^{\epsilon_0} \!M_{\alpha,i} \asymp {}^{\epsilon_0} \!M_{\beta,i}$, $\alpha \neq \beta$ with $\alpha, \beta \in \{1,\ldots \kappa\}$ and ${}^{\epsilon_0} \!F_i \asymp {}^{\epsilon_0} \!M_{\beta,i}$, the capabilities of ${}^{\epsilon_0} \!A_i$ are modeled by the $\epsilon_0$-NFA ${}^{\epsilon_0} \!K_i$ as:
\begin{equation}
\label{eq:agents_cap}
\small{{}^{\epsilon_0} \!K_i \triangleq \Bigg\{ \mathop{\bigcup}_{\beta=1}^{\kappa} {_\asymp}  \ {}^{\epsilon_0} \!M_{\beta,i} \Bigg\} \backslash_\asymp {}^{\epsilon_0} \!F_i .}
\end{equation}

\subsubsection{Individual Agent Constraints}
Let $\lambda$ be the total number of individual constraints of ${}^{\epsilon_0} \!A_i$ and ${}^{\epsilon_0} \!N_{\xi,i}$ the $\xi$'th individual constraint of ${}^{\epsilon_0} \!A_i$, where $\xi \in \{1,\ldots \lambda\}$.

We model the agent's constraints as the union of individual agent constraints utilizing the union of compatible automata operation. 
\subsubsection{Agent Constraints}
Considering $\lambda$ compatible $\epsilon_0$-NFAs such that ${}^{\epsilon_0} \!N_{\xi,i} \asymp {}^{\epsilon_0} \!N_{\eta,i}$, $\xi \neq \eta$ with $\eta \in \{1,\ldots \lambda\}$, the constraints of ${}^{\epsilon_0} \!A_i$ are modeled by the $\epsilon_0$-NFA ${}^{\epsilon_0} \!D_i$ as:
\begin{equation}
\small{{}^{\epsilon_0} \!D_i \triangleq \Bigg\{ \mathop{\bigcup}_{\xi=1}^{\lambda} {_\asymp} \ {}^{\epsilon_0} \!N_{\xi,i} \Bigg\}.}
\label{eq:agents_constr}
\end{equation}

%Agent model
Considering ${}^{\epsilon_0} \!K_i \asymp {}^{\epsilon_0} \!D_i$, the agent $i^{th}$ agent is modeled by the $\epsilon_0$-NFA of ${}^{\epsilon_0} \!K_i$ after subtracting ${}^{\epsilon_0} \!D_i$:
\begin{equation}
{}^{\epsilon_0} \!A_i \triangleq {}^{\epsilon_0} \!K_i \ \backslash_\asymp \ {}^{\epsilon_0} \!D_i .
\label{eq:agent}
\end{equation}

\subsection{Environment Model}
\label{section:environment_model}
Agents are acting in the environment to reach individual states while being influenced by state capabilities and constraints relating to other agents. Agents who rely on other agents to perform actions or reach goals are grouped into a team. The capabilities and constraints of the team are modeled as a combination of individual states of the members of the team. These capabilities and constraints, called inter-agent capabilities and constraints, express the allowed and not-permitted environment state transitions of those agents. Given the agents' capabilities and constraints and the inter-agent capabilities and constraints, the environment model captures all the possible combinations of agents' states. The procedure of constructing agents' and environment models is illustrated in Fig.~\ref{fig:flow_diagram}, as well.

\begin{figure}[ht]
\centering
\includegraphics[scale=0.15]{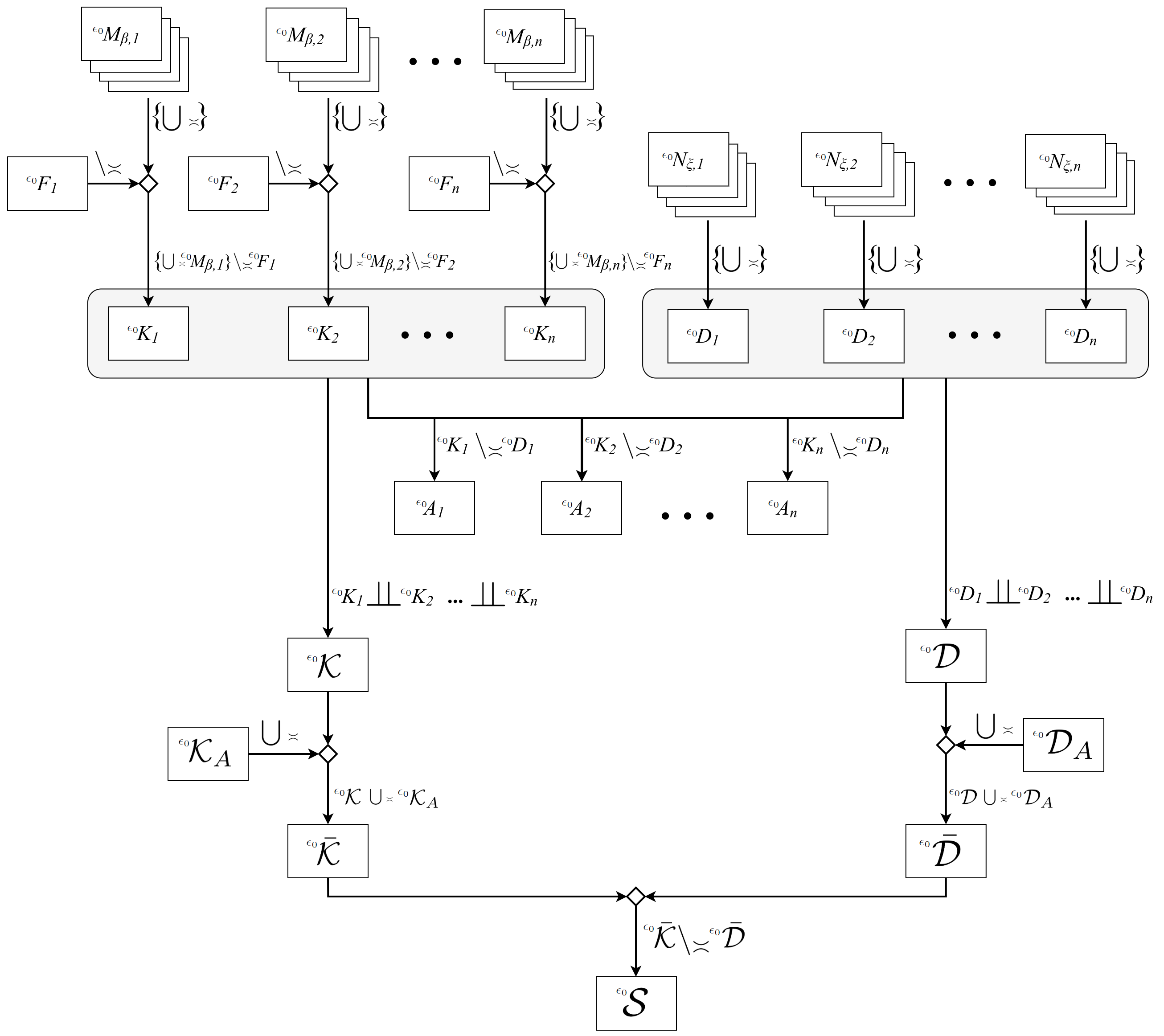}
\caption{Information flow diagram for construction of the agents' and environment models $\epsilon_0$-NFAs.}
\label{fig:flow_diagram}
\end{figure}

We model the environmental capabilities as the concatenation of $\epsilon_0$-NFAs formed as in eq. \ref{eq:agents_cap} utilizing the concatenation operation of compatible automata to express the allowed environment state transitions.
\subsubsection{Environmental Capabilities}
For the $\epsilon_0$-NFAs ${}^{\epsilon_0} \!K_i$ and ${}^{\epsilon_0} \!K_j$, $i \neq j$ with $i, j \in \{1,\ldots n\}$, let ${}^{\epsilon_0} \!K_i \asymp {}^{\epsilon_0} \!K_j$. The environmental capabilities are modeled by the $\epsilon_0$-NFA ${}^{\epsilon_0} \!\mathcal{K}$ as:
\begin{equation}
{}^{\epsilon_0} \!\mathcal{K} \triangleq \mathop{\indep _\asymp}_{i=1}^{n} {}^{\epsilon_0} \!K_i.    
\label{eq:environmental_caps}
\end{equation}

We model the environmental constraints as the concatenation of $\epsilon_0$-NFAs formed as in eq. \ref{eq:agents_constr} utilizing the concatenation operation of compatible automata to express the not-permitted environment state transitions.
\subsubsection{Environmental Constraints}
For the $\epsilon_0$-NFAs ${}^{\epsilon_0} \!D_i$ and ${}^{\epsilon_0} \!D_j$, $i \neq j$ with $i, j \in \{1,\ldots n\}$, let ${}^{\epsilon_0} \!D_i \asymp {}^{\epsilon_0} \!D_j$. The environmental constraints are modeled by the $\epsilon_0$-NFA ${}^{\epsilon_0} \mathcal{D}$ as:
\begin{equation}
{}^{\epsilon_0} \mathcal{D} \triangleq \mathop{\indep _\asymp}_{i=1}^{n} {}^{\epsilon_0} \!D_i.    
\label{eq:environmental_cons}
\end{equation}

Then, inter-agents capabilities $\epsilon_0$-NFA and inter-agents constraints $\epsilon_0$-NFA are modeled utilizing the agents' models formed as in eq.\ref{eq:agent} as follows:
\subsubsection{Inter-Agents Capabilities}
\label{def:inter_caps}
Let $\mathrm{co} \subseteq {}^{\epsilon_0} \!\mathcal{A}$ with $\|\mathrm{co}\|  \le n $ be a set of compatible agent $\epsilon_0$-NFAs. Then, the $\epsilon_0$-NFA ${}^{\epsilon_0} \!\mathcal{K}_A \subset \mathop{\indep _\asymp}_{i\in \mathrm{co}} {}^{\epsilon_0} \!A_i$ denotes the inter-agents capabilities\footnote{Inter-Agents Failure Modes can be defined so as to restrict Inter-Agents Capabilities} between the members of $\mathrm{co}$.

\subsubsection{Inter-Agents Constraints}
\label{def:inter_cons}
Let $\mathrm{co} \subseteq {}^{\epsilon_0} \!\mathcal{A}$ with $\|\mathrm{co}\|  \le n $ be a set of compatible agent $\epsilon_0$-NFAs. Then, the $\epsilon_0$-NFA ${}^{\epsilon_0} \mathcal{D}_A \subset \mathop{\indep _\asymp}_{i\in \mathrm{co}} {}^{\epsilon_0} \!A_i $ denotes the inter-agents constraints between the members of $\mathrm{co}$.

We model the global capabilities as the $\epsilon_0$-NFA of environmental capabilities (as formed in eq.\ref{eq:environmental_caps}) utilizing the union operation of environmental capabilities with inter-agent capabilities.
\subsubsection{Global Capabilities}
For the $\epsilon_0$-NFAs ${}^{\epsilon_0} \!\mathcal{K}$ and ${}^{\epsilon_0} \mathcal{K}_A$, let ${}^{\epsilon_0} \!\mathcal{K} \asymp {}^{\epsilon_0} \mathcal{K}_A$. The global capabilities are modeled by the $\epsilon_0$-NFA ${}^{\epsilon_0} \!{\widetilde{\mathcal{K}}}$ defined as:
\begin{equation}
{}^{\epsilon_0} \!{\widetilde{\mathcal{K}}} \triangleq {}^{\epsilon_0} \!\mathcal{K} \cup_{\asymp} {}^{\epsilon_0} \mathcal{K}_A.   
\label{eq:global_caps}
\end{equation}

We model the global constraints as the $\epsilon_0$-NFA of environmental constraints (as formed in eq.\ref{eq:environmental_cons}) utilizing the union operation of environmental constraints with inter-agent constraints.
\subsubsection{Global Constraints}
For the $\epsilon_0$-NFAs ${}^{\epsilon_0} \mathcal{D}$ and ${}^{\epsilon_0} \!\mathcal{D}_A$, let ${}^{\epsilon_0} \mathcal{D} \asymp {}^{\epsilon_0} \!\mathcal{D}_A$. The global constraints are modeled by the $\epsilon_0$-NFA ${}^{\epsilon_0} {\widetilde{\mathcal{D}}}$ defined as:
\begin{equation}
{}^{\epsilon_0} {\widetilde{\mathcal{D}}} \triangleq  {}^{\epsilon_0} \mathcal{D} \cup_{\asymp} {}^{\epsilon_0} \!\mathcal{D}_A.
\label{eq:global_cons}
\end{equation}

%Environment model
Considering ${}^{\epsilon_0} \!\widetilde{\mathcal{K}} \asymp {}^{\epsilon_0} \widetilde{\mathcal{D}}$, the environment model is constructed by the $\epsilon_0$-NFA of ${}^{\epsilon_0} \!\widetilde{\mathcal{K}}$ after subtracting ${}^{\epsilon_0} \widetilde{\mathcal{D}}$: 
\begin{equation}
{}^{\epsilon_0} \!\mathcal{S} \triangleq {}^{\epsilon_0} \!\widetilde{\mathcal{K}} \backslash_\asymp {}^{\epsilon_0} \widetilde{\mathcal{D}}
\label{eq:environment}
\end{equation}
where the cardinality of $X_\mathcal{S}$ is $\theta=\prod_{i=1}^n |X_{A_i}|$. 

\subsection{Task Specification}
Let $x \in X_\mathcal{S}$ and let $b$ be an $n$-bit binary indicating the states of the agents that we are interested in specifying. The task specification is a projection $\gamma \triangleq proj(x,b)$ that indicates the desired state of specific agents in the system.

We can now proceed to formally state the following.
\subsection{Problem statement}

Given a set of individual and inter-agents capabilities and constraints including their failure modes, the initial state of the system and a task specification, determine a string that provides an optimal execution\footnote{Automata execution (or run) as defined in e.g. \cite{cassandras2009introduction}} (in terms of total transition cost) that brings the system from any initial state to any state satisfying the task specification. 

\section{FORMULATION AS A MODULE COMPOSITION PROBLEM}
\label{sec:MC_PF}
Let us now consider the environment model ${}^{\epsilon_0} \! \mathcal{S}$ and assume that we would like to use  $x_{0,\mathcal{S}} \in X_\mathcal{S}$ as the initial state of our system. Thus, $\mathcal{S} = \Delta({}^{\epsilon_0} \! \mathcal{S}, x_{0,\mathcal{S}})$ is the DFA description of the environment model. Let $P_{in} \in X_\mathcal{S}$ be the non-empty finite set of input ports and $P_{out} \in X_\mathcal{S}$ be the non-empty finite set of output ports, where $P = P_{in} \cup P_{out}$ and $P_{in} \cap P_{out} \ne \emptyset$. The module $\Tau_j$ is defined as:
\begin{equation}
\begin{gathered}
\Tau_j \triangleq \{ p_j, e_j, q_j \} 
\end{gathered}
\end{equation}
where $p_j \in P_{in}$, $e_j \in E_\mathcal{S}$ and $q_j \in P_{out}$. 

Define the $i^{th}$ task module  as $\Tau_{0,i} = \left\{x_{0,\mathcal{S}}, e_{0,i}, x_{d,i} \right\}$ where $proj(x_{d,i},b) = \gamma$, $x_{d,i} \in X_\mathcal{S}$ and $e_{0,i}$ a virtual transition from the initial to the final state. Observe that there are $|\gamma|!$ potential solutions so $i\in \left\{1,\ldots |\gamma|! \right\}$. 

To tackle the problem defined in the Problem Statement, we proceed to formulate our problem as a Module Composition Problem (MCP) \cite{tripakis2003automated}. Since we are using single-port modules, we will have a special case of the MCP that is solvable in polynomial time. We define $\mathcal{T}_i$ as the finite closed module chain containing $\Tau_{0,i}^{-1}$ describing the sequential environment states transitions during the execution (Fig. \ref{fig:chain-example}) defined as:
\begin{equation}
\mathcal{T}_i = \{\Tau_{0,i}^{-1},\Tau_{1,i},\ldots \Tau_{z,i} \}
\end{equation}
 where $z_i=|\mathcal{T}_i|$. In the sequel we will focus on the case where $|\gamma | =1$ and drop the index $i$.
\begin{figure}[ht]
\centering
\includegraphics[scale=0.23]{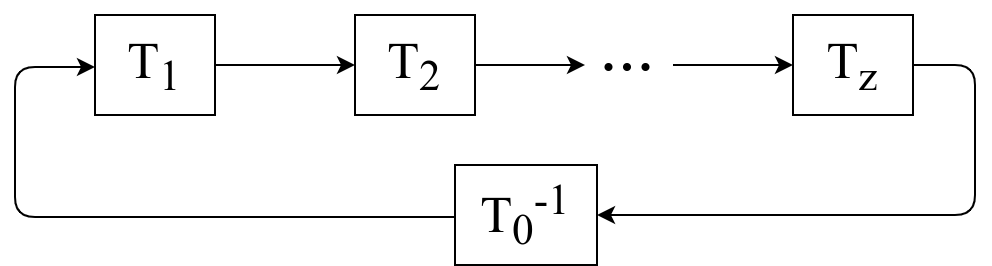}
\caption{Closed module chain $\mathcal{T}$ with sequential single-port modules.}
\label{fig:chain-example}
\end{figure}

\iffalse
Hence the discrete optimization problem could be formed as shown in Fig.\ref{fig:diagram-singleport}:
\begin{figure}[ht]
\centering
\includegraphics[scale=0.14]{diagram-singleport_latest1.png}
\caption{Discrete optimization problem formed as a closed chain $\mathcal{T}_i$.}
\label{fig:diagram-singleport}
\end{figure}
\fi

Let $t_j$ denote the number of instances of $\Tau_j$, $c_j=c(T_j)$ is the cost of implementing module $T_j$, then the discrete optimization problem can be posed as an integer programming problem as follows:
\vspace{-.1in}
\begin{equation*}
min \sum_{j \in \left\{1,\ldots \left|E_{\mathcal{S}} \right|  \right\}} c_j t_j 
\vspace{-.1in}\end{equation*}
subject to:
\vspace{-.1in}\begin{subequations}\label{global}
\begin{align*}
\label{particular-a}
& t_0=1, t_j \in \mathbb{Z}^*  \\
& \sigma_p = \sum_{q \in P_{out}} w_{p,q}, \forall p \in P_{in} \\
& \mu_q = \sum_{p \in P_{in}} w_{p,q}, \forall q \in P_{out}
\end{align*}
\end{subequations}
where $\mathbb{Z}^*$ denotes the non-negative integers, $w_{p, q}$ the number of connections between input port $p$ and output port $q$, $\sigma_p$ the number of modules with input port $p$ utilized and $\mu_q$ the number of modules with output port $q$ utilized. Since in the current work we have implemented only single port modules, the above integer programming problem can be reduced to the shortest directed path problem \cite{Tripakis2003}, that can be solved utilizing the Dijkstra's shortest dipath algorithm \cite{dijkstra1959note}. The optimal solution composes the $\mathcal{T}$ that minimizes the cost of states transitions in order to fulfil the task specification $\gamma$. 

The drawback of using singe port modules in the current work is that we are only limiting transitions to be performed by one agent at a time. Multi-port modules are currently being considered as a further research topic, to enable multiple agents to perform concurrent transitions and is beyond the scope of the current work. We have to also note here that the additional effort in casting the problem as a module composition one, enables us to seamlessly use the generated module chain as a model system for building supervisory controllers as e.g. in \cite{loizou2005automated}, and is currently under active research.

\section{ANALYSIS}
\label{sec:Analysis}
\subsection{Task Planning Processes}

In this section, we present the properties and process of SPECTER task planner to compute the optimal plan based on the agents' and environment models. The SPECTER's process includes the pre-processing phase and the problem solving phase.

The pre-processing phase includes the construction of the agents models (\ref{eq:agent}), the environment model (\ref{eq:environment}) (see Fig.\ref{fig:flow_diagram}) and the graph of ${}^{\epsilon_0} \!\mathcal{S}$ using adjacency matrix representation. The computational time to construct ${}^{\epsilon_0} \! A_i$ is $O(|X_{A_i}|^2)$. The time required to create ${}^{\epsilon_0} \!\mathcal{S}$ is $O(|{X_\mathcal{S}|}^3)$. ${}^{\epsilon_0} \!\mathcal{S}$ is converted to a weighted directed graph $\mathcal{H_\mathcal{S}}=(\mathcal{V_\mathcal{S}}, \mathcal{E_\mathcal{S}})$, where the set of nodes $\mathcal{V_\mathcal{S}}$ corresponds to the set of states $X_{\mathcal{S}}$, the set of edges $\mathcal{E_\mathcal{S}}$ is defined by $f_\mathcal{S}$ associated with its cost $g_\mathcal{S}$. The adjacency matrix representation of $\mathcal{H_\mathcal{S}}$ is a 2-dimensional array $X_{\mathcal{S}} \times X_{\mathcal{S}}$. Each element in the array stores the cost $g_\mathcal{S}$ related to the edge $f_\mathcal{S}(x \in X_{\mathcal{S}}, e \in E_{\mathcal{S}})$. The amount of space required to store the array is $O(\left| \mathcal{V_{\mathcal{S}}} \right| ^2)$ in worst case. Computational efficiency can be succeeded if the pre-processing computations are made beforehand. The constructed environment model can then be used to solve all the possible  task planning problems by arbitrarily choosing the initial and target states (projection) without the need to reconstruct the environment model.

The problem solving phase encapsulates the synthesis of the optimal task plan and the integration of individual agent failure mode. The Dijkstra's algorithm is implemented over the weighted graph $\mathcal{H_\mathcal{S}}$ to find the optimal task plan $\mathcal{T}$ as given in Algorithm \ref{alg:algo-modulechain}. In the case where a fault is detected for the transition from state $x_i$ to $x_j$ of $\mathcal{S}$ (e.g. by some fault identification system) that affects agent $\nu$, we can disable all the affected transitions by finding the set of states ${X_i}'$, ${X_j}'$, where $proj(x_i, b_\nu) \equiv proj({x_i}' \in {X_i}', b_\nu)$ and $proj(x_j, b_\nu) \equiv proj({x_j}' \in {X_j}', b_\nu)$. Then, we eliminate the transitions from all ${x_i}' \in {X_i}' $ to all ${x_j}' \in {X_j}'$. The above procedures incorporates on-the-fly a possible new failure mode into $\mathcal{S}$ without the need to repeat the costly pre-processing phase. The computational time required for this modification is $\theta'= \prod_{i=1, \ i \neq \nu}^{n-1} |X_{A_i}|$ in the worst case. 

The task planning problem of minimizing the length of the plan is NP-hard \cite{barbehenn1998note}. Let $\mathcal{P}$ be the solution to this problem found after running Dijkstra's algorithm and let $\mathcal{P}_i$ denote it's $i$'th element, $i\in \left\{1,\ldots \left|\mathcal{P} \right| \right\}$. In Algorithm \ref{alg:funcomplete}, the optimal solution can then be found by running the algorithm for all states that satisfy the task specification, that is $\frac{\prod_{i=1}^{n} | X_{A_i}| }{\left| X_{A_{\sigma}}\right|}$ times in the worst case scenario, where $\sigma$ denotes the agent index that was used for the task specification $\gamma$. The running time of Dijkstra's algorithm implementing the Complete Function of Algorithm \ref{alg:algo-modulechain} is $O((\mathcal{V_\mathcal{S}} + \mathcal{E_\mathcal{S}}) log\mathcal{V_\mathcal{S}})$.

\begin{algorithm}[ht]
\centering
\caption{Create module chain $\mathcal{T}$.}
\begin{algorithmic}[1]
 \Require $\epsilon_0$-NFA ${}^{\epsilon_0} \! \mathcal{S}$, initial state $x_{0,\mathcal{S}}$, task specification $\gamma$, solution type $\text{ST}$ (``Complete'' or ``Heuristic'')
  \Ensure Module chain $\mathcal{T}$
  \State Initialize $\mathcal{H}$ to be a zero matrix
 \State Initialize $E$ to be an empty cell array
 \State Initialize $\mathcal{T}$ to an empty set
 \State $\mathcal{S} \longleftarrow \Delta({}^{\epsilon_0} \! \mathcal{S}, x_{0,\mathcal{S}})$
 \If  {$proj(x_{0,\mathcal{S}},b) =  \gamma \wedge x_{0,\mathcal{S}} \in X_{m,\mathcal{S}}$ }
    \State \Return $\mathcal{T}$
 \Else
 \For {$i \in \left\{1, \ldots \left|X_\mathcal{S}\right| \right\}$}
    \For {$j \in \left\{1, \ldots \left|X_\mathcal{S}\right| \right\}$}
        \If {$\exists e\in E_\mathcal{S}: f_{\mathcal{S}}(x_i,e)=x_j$}
            \State $\mathcal{H}[i][j] \longleftarrow g_{\mathcal{S}}(e)$
            \State $E\{i\}\{j\} \longleftarrow e$
        \EndIf
    \EndFor
 \EndFor
  \If {$\text{ST} = $ ``Complete''}
    \State [$\mathcal{P}$, cost] $\longleftarrow$ \Call{Complete}{$\mathcal{H}, \mathcal{S}$, $x_{0,\mathcal{S}}$, $\gamma$} 
 \Else
    \State [$\mathcal{P}$, cost] $\longleftarrow$ \Call{Heuristic}{$\mathcal{H}, \mathcal{S}$, E, $x_{0,\mathcal{S}}$, $\gamma$}
 \EndIf
 \For {$ i \in \left\{1, \ldots \left|\mathcal{P}\right|-1 \right\} $}
    \State $\mathcal{T} \longleftarrow \mathcal{T} \cup \left\{ \mathcal{P}_i,\  E\{\mathcal{P}_{i}\}\{\mathcal{P}_{i+1}\},\  \mathcal{P}_{i+1} \right\}$
 \EndFor
 \State \Return $\mathcal{T}$, cost
\EndIf
\end{algorithmic}
\label{alg:algo-modulechain}
\end{algorithm}

\begin{algorithm}[ht]
\centering
\caption{Complete Function.}
\begin{algorithmic}[1]
\Function{Complete}{$\mathcal{H}, \mathcal{S}$, $x_{0,\mathcal{S}}$, $\gamma$}
 \State Initialize $cost_{min} \longleftarrow \infty$
 \For {$i \in \left\{1, \ldots \left|X_{\mathcal{S}} \right| \right\}$}
    \For {$j \in \left\{1, \ldots \left|X_{\mathcal{S}} \right| \right\}$}
        \If {$proj(x_j, b) = \gamma \wedge x_j \in X_{m, \mathcal{S}}$}
            \State [$\mathcal{P}$, cost] $\longleftarrow$ Dijkstra($\mathcal{H}, x_{0,\mathcal{S}}, x_j$)
            \If {$\mathcal{P} = \emptyset$}
                \State \Return Task infeasible.
            \EndIf
            \If {cost $< cost_{min}$}
                \State $cost_{min} \longleftarrow $ cost
                \State $\mathcal{P}_{optimal} \longleftarrow \mathcal{P}$
            \EndIf
        \EndIf
    \EndFor
 \EndFor
 \State \Return $\mathcal{P}_{optimal}$, $cost_{min}$
\EndFunction
\end{algorithmic}
\label{alg:funcomplete}
\end{algorithm}

While this might be feasible for problems of relatively small size, a heuristic approach is proposed to reduce the computational time required to find a sub-optimal solution. The proposed heuristic, implemented in Algorithm \ref{algo:funheuristic}, maintains the complexity to the one of Dijkstra's algorithm, while sacrificing optimality and completeness. The computational time of the proposed heuristic is $O( \mathcal{E_{\mathcal{S}}} log \mathcal{V_\mathcal{S}})$. We track the solution to the minimum element $k$ where $proj(\mathcal{P}_k,b)=\gamma$ and use the solution $\mathcal{P}^* = \left\{ \mathcal{P}_1,\ldots \mathcal{P}_k \right\}$. While there are some cases where the proposed heuristic recovers optimality, a classification of those cases and the conditions for feasibility of solutions is currently under consideration but is not part of the current work. The intuition here is to set the goal state $x_d$ to be such that $proj(x_d,b)= \gamma$ and $proj(x_d,\overline b) = proj(x_{0,\mathcal{S}},\overline b)$ where $\overline b$ denotes the bitwise negation. 

\begin{algorithm}[ht]
\centering
\caption{Heuristic Function.}
\begin{algorithmic}[1]
\Function{Heuristic}{$\mathcal{H}, \mathcal{S}$, E, $x_{0,\mathcal{S}}$, $\gamma$}
 \State Initialize $x_d$ to an empty set
 \State Find $x_d \in X_{m, \mathcal{S}} : proj(x_d,b) = \gamma \wedge proj(x_d,\overline{b}) = proj(x_{0,\mathcal{S}},\overline{b}) $
 \State Initialize $c^* \longleftarrow 0$
 \State [$\mathcal{P}$, cost] $\longleftarrow$ Dijkstra($\mathcal{H}, x_{0,\mathcal{S}}, x_d$)
 \If {$\mathcal{P} = \emptyset$}
    \State \Return No path to $x_d$.
 \EndIf
 \For {$k \in \left\{2, \ldots \left| \mathcal{P} \right| \right\}$}
       \State $c^* \longleftarrow c^* + g_\mathcal{S} (E\{\mathcal{P}_{k-1}\}\{\mathcal{P}_{k}\})$
        \If {$proj(\mathcal{P}_k,b) = \gamma$}
            \State $\mathcal{P}^* \longleftarrow \left\{ \mathcal{P}_1,\ldots \mathcal{P}_k \right\} $    
            \State \Return $\mathcal{P}^*, c^*$
        \EndIf
    \EndFor
\EndFunction
\end{algorithmic}
\label{algo:funheuristic}
\end{algorithm}

\subsection{Completeness and Optimality}

We have the following results regarding the completeness of Algorithms \ref{alg:algo-modulechain}, \ref{alg:funcomplete}:
\begin{proposition}
\label{prop:complete}
For a system constructed based on individual and inter-agent capabilities, constraints and failure modes of the multi-agent system as presented in Section \ref{sec:PROB_FORM}, the resulting environment model of eq. \ref{eq:environment} is complete, in the sense that it represents all and only those state transitions that are dictated by the automata capturing the individual and inter-agent capabilities, constraints and failure modes.
\end{proposition}
\begin{proof}
To demonstrate the above claim we need to show that during the composition of ${}^{\epsilon_0} \!\mathcal{S}$, no valid transitions or states are being removed from the system and no new states or transitions are being introduced.

\begin{enumerate}[label=(\alph*)]
\item No states removed or added.

This can be shown by observing that from Corollary \ref{cor:ResultNFA} the $\epsilon_0$-NFAs is closed under the operations of union, subtraction and concatenation. The union operation by Definition \ref{def:union} operates only on states defined in its arguments and the resulting states are the union of the argument's states that are part of agents' capabilities and constraints. The subtraction operation by Definition \ref{def:subtraction}
does not remove any states from the subtrahend argument. The concatenation by Definition \ref{def:concatenation} operation creates the cross product state space of its argument. None of the operations introduces any state that does not exist in its arguments.

\item No valid events are removed and no new transitions are introduced.

Union and concatenation operations: The resulting event set is the union of the operator argument events. No new events that do not exist in the arguments are introduced. 

Subtraction operation: The resulting event set includes only the events that are in the subtrahend and not in the subtractor's event set without introducing any new event. In particular failure mode events are removed from individual agent's capabilities with subtraction operation without affecting any other events. The event set of global constraints are removed from global capabilities without affecting events that are not included in the global capabilities event set.

Any event that is in the individual or inter-agent capabilities and is not in the failure modes, individual constraints or inter-agent constraints event sets will be included in the environment model. Any event that is in the failure modes event set (but not in the inter-agent capabilities) as well as any event in the  individual or inter-agent constraint event set, will not appear in the environment model event set. No events that are not included in the individual and inter-agent capabilities will appear in the environment model event set. 
\end{enumerate}
Hence all and only transitions that are valid will appear in the environment model's event set.
\end{proof}

With the completeness property in place we can now state the following result about the optimality properties of the proposed system: 
\begin{proposition}
\label{prop:optimal}
For a system constructed based on individual and inter-agent capabilities, constraints and failure modes of the multi-agent system as presented in Section \ref{sec:PROB_FORM},  and assuming that only one agent is allowed to operate at any time instant, the solution of Algorithm \ref{alg:algo-modulechain} with the "Complete" solution type of Algorithm \ref{alg:funcomplete}, produces the optimal sequence of actions to achieve task plan $\gamma$. 
\end{proposition}
\begin{proof}
Since according to Proposition \ref{prop:complete} the environment model is complete, then all (if any) optimal solutions are encoded in the transition system imposed by the transition function of the  environment model. The module composition problem is an integer optimization problem and in our case maps the task planning problem to the shortest directed path - a graph search problem. The implemented Dijkstra's solution is both complete and is guaranteed to find an optimal solution if such a solution exists.
\end{proof}

\section{CASE STUDIES}
\label{sec:Exp}
The proposed methodology could be applied in various applications, such as manufacturing logistics. To demonstrate the applicability of the proposed methodology, two case studies are presented related to a manufacturing logistics workflow, implemented on a computer with AMD Ryzen 5 4500U 2.3 GHz and 16GB RAM. In first case study, we consider a non-trivial logistic problem of appropriate size that permits an effective illustration with ${}^{\epsilon_0} \!\mathcal{S}$ that has $\prod_{i=1}^{n} |X_{A_i}| = 560$ states. The second case study pertains to a pattern of logistic workflows on manufacturing. This scenario is considered to face the challenges of productivity, efficiency and effectiveness of internal logistics management and agent automation in the workflows. In this case, ${}^{\epsilon_0} \!\mathcal{S}$ has $\prod_{i=1}^{n} |X_{A_i}| = 1.875 \times 10^6$ states. In both case studies, the module chains $\mathcal{T}$ provided either by the “Complete” or “Heuristic” options are identical and the edges of automata are associated with time costs in seconds.

\subsection{Case Study I}
We consider the team of $n=4$ agents with ${}^{\epsilon_0} \!\mathcal{A}=\{{}^{\epsilon_0} \!A_1, {}^{\epsilon_0} \!A_2, {}^{\epsilon_0} \!A_3, {}^{\epsilon_0} \!A_4 \}$, where ${}^{\epsilon_0} \!A_1$, ${}^{\epsilon_0} \!A_2$, ${}^{\epsilon_0} \!A_3$, ${}^{\epsilon_0} \!A_4$ denote the robots $R_1$ and $R_2$, the human $W_1$ and the item $I_1$, respectively. 
\begin{figure}[ht]
\centering
\includegraphics[scale=0.22]{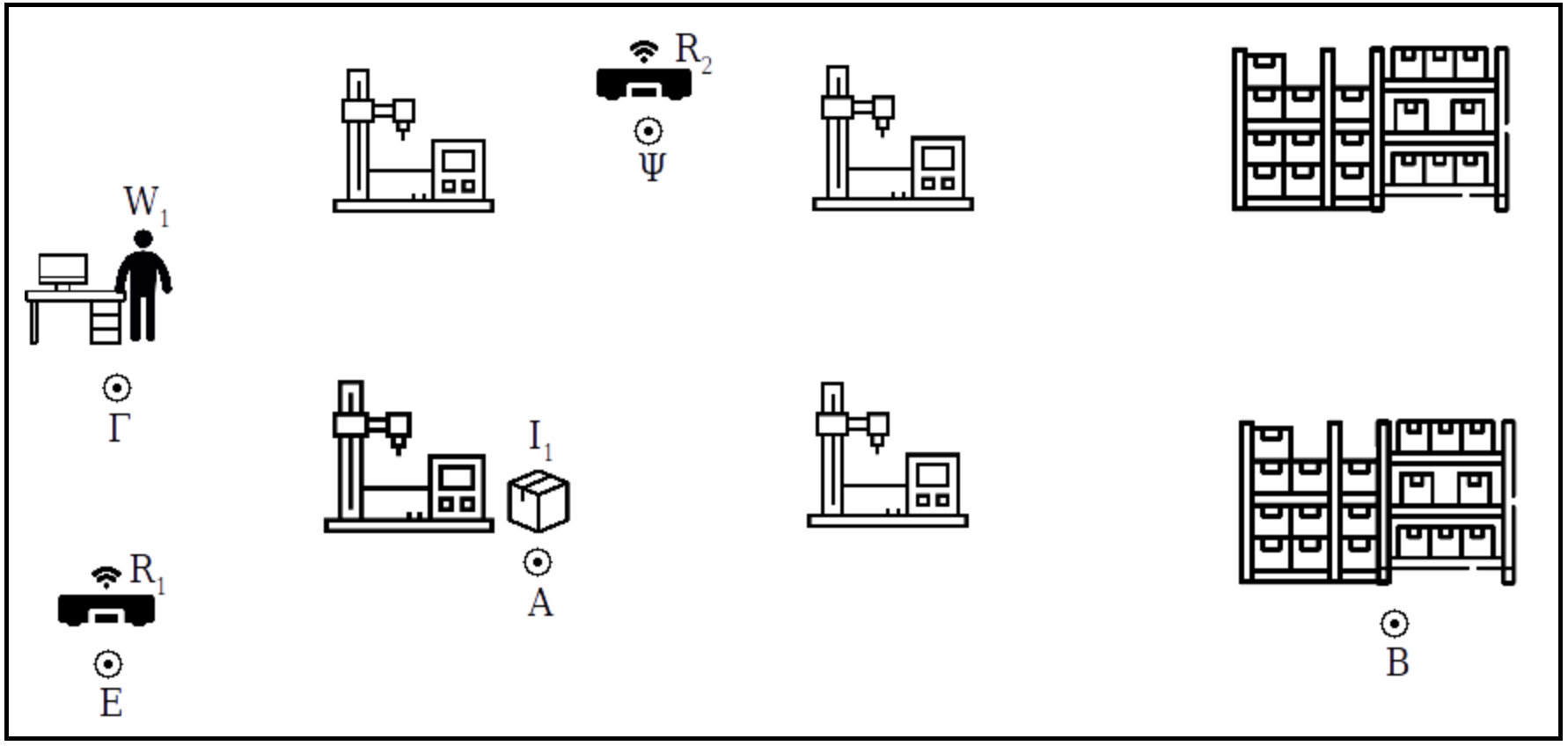}
\caption{Mock-up factory plant and agents.}
\label{fig:setup-example}
\end{figure}
Agents are on their initial state as shown in the plant of Fig. \ref{fig:setup-example}. $\text{E}$ and $\Psi$ denote the docking stations of robots $R_1$ and $R_2$, respectively. $W_1$ is located at work-cell $\Gamma$ and $I_1$ is at $\text{A}$. The objective task is to transport $I_1$ to $B$ in minimum time. Automata that are missing from the case study are considered as empty automata. Robot coordination and navigation is handled by appropriate controllers whose description is beyond the scope of the current work.

We consider that $X_{A_1} \cap X_{A_2} \cap X_{A_3} \cap X_{A_4} = \left\{A, B, \Gamma \right\}$. $R_1$ and $R_2$ denote the payload of ${}^{\epsilon_0} \!A_1$ and ${}^{\epsilon_0} \!A_2$ in $X_{A_4}$ (i.e. $I_1$ can be loaded to $R_1$ or $R_2$). We define the capabilities and the constraints of ${}^{\epsilon_0} \!A_1, {}^{\epsilon_0} \!A_2, {}^{\epsilon_0} \!A_3, {}^{\epsilon_0} \!A_4$ as illustrated in Fig.~\ref{fig:agents_caps} and~\ref{fig:agents_cons} using equations (\ref{eq:agents_cap}) and (\ref{eq:agents_constr}). Subtracting ${}^{\epsilon_0} \!D_1$ from ${}^{\epsilon_0} \!K_1$ and ${}^{\epsilon_0} \!D_2$ from ${}^{\epsilon_0} \!K_2$, we generate ${}^{\epsilon_0} \!A_1$ and ${}^{\epsilon_0} \!A_2$. Since ${}^{\epsilon_0} \!D_3, {}^{\epsilon_0} \!D_4 = \emptyset$, ${}^{\epsilon_0} \!A_3$ and ${}^{\epsilon_0} \!A_4$ are derived from ${}^{\epsilon_0} \!K_3$ and ${}^{\epsilon_0} \!K_4$ respectively. Using the concatenation of ${}^{\epsilon_0} \!K_1$, ${}^{\epsilon_0} \!K_2$, ${}^{\epsilon_0} \!K_3$ and ${}^{\epsilon_0} \!K_4$, we generate the ${}^{\epsilon_0} \!\mathcal{K}$ $\epsilon_0$-NFA. Using the concatenation of ${}^{\epsilon_0} \!D_1$ and ${}^{\epsilon_0} \!D_2$, we extract the ${}^{\epsilon_0} \mathcal{D}$ $\epsilon_0$-NFA. In Fig.\ref{fig:interagents_caps}, ${}^{\epsilon_0} \!\mathcal{K}_A$ states the loading and unloading actions and defines the inter-agent capabilities between $R_1$, $R_2$, $W_1$, $I_1$. Adding ${}^{\epsilon_0} \!\mathcal{K}_A$ to ${}^{\epsilon_0} \mathcal{K}$, we construct ${}^{\epsilon_0} \mathcal{\widetilde{K}}$. Considering that ${}^{\epsilon_0} \mathcal{D}_A = \emptyset$, then ${}^{\epsilon_0} \!\mathcal{\widetilde{D}} = {}^{\epsilon_0} \!\mathcal{D}$. Finally, the model of ${}^{\epsilon_0} \!\mathcal{S}$ is given by subtracting ${}^{\epsilon_0} \mathcal{\widetilde{D}}$ from ${}^{\epsilon_0} \!\mathcal{\widetilde{K}}$ using equation (\ref{eq:environment}), where ${}^{\epsilon_0} \!X_\mathcal{S} = \prod_{i=1}^{n} |X_{A_i}| = 560$ states. We state $\gamma$ as $proj(x_d, b)=\text{B}$, where $b=0001$. The task planning problem here is to find the optimal $\mathcal{T}$ that takes $proj(x_{0,\mathcal{S}},b)=\text{A}$ to $\gamma$ utilizing the robot agent that minimizes the transition cost for the task execution subject to given environment capabilities and constraints. 
\begin{figure}[ht]
\centering
\includegraphics[scale=0.25]{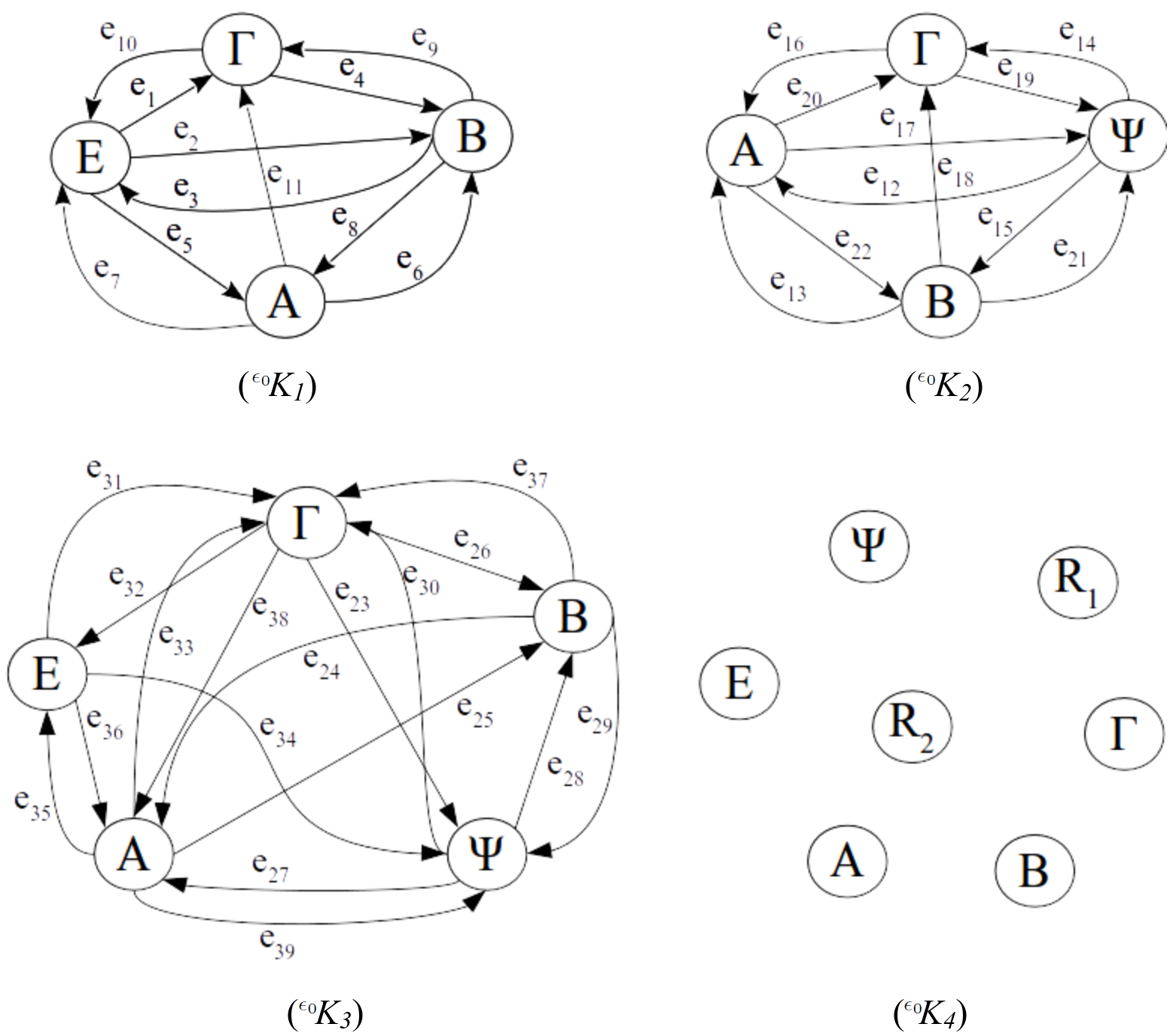}
\caption{State transition graphs of agents' capabilities.}
\label{fig:agents_caps}
\end{figure}

\begin{figure}[ht]
\centering
\includegraphics[scale=0.25]{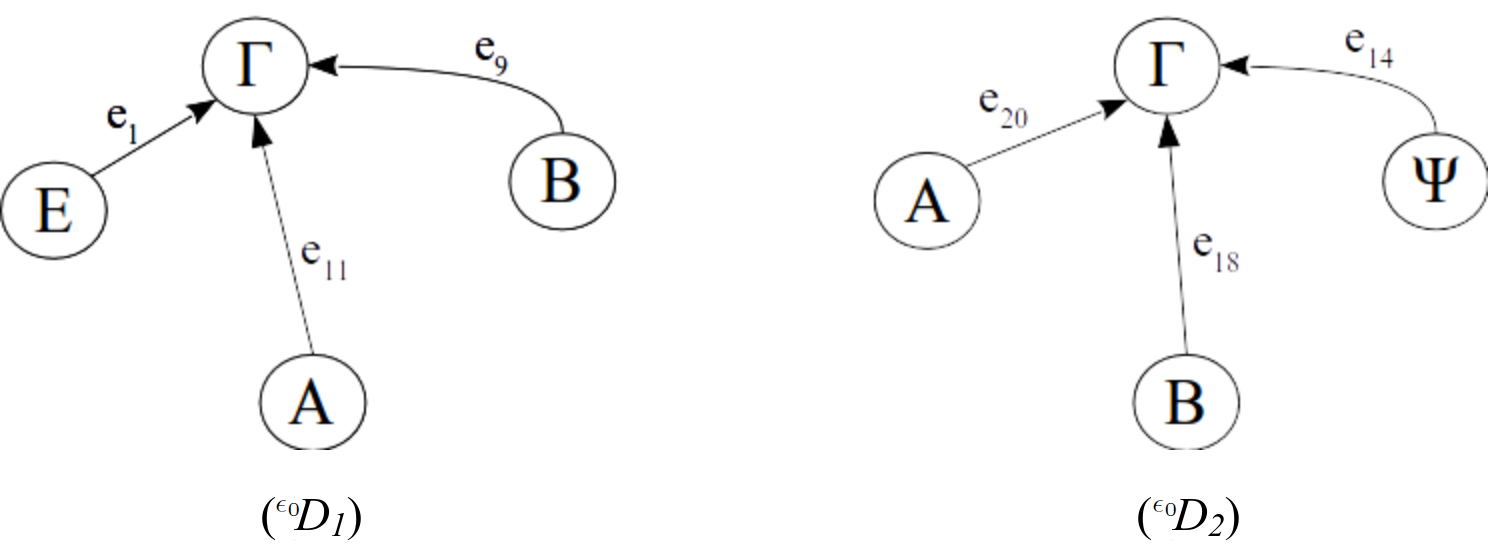}
\caption{State transition graphs of agents' constraints.}
\label{fig:agents_cons}
\end{figure}

\begin{figure}[ht]
\centering
\includegraphics[scale=0.25]{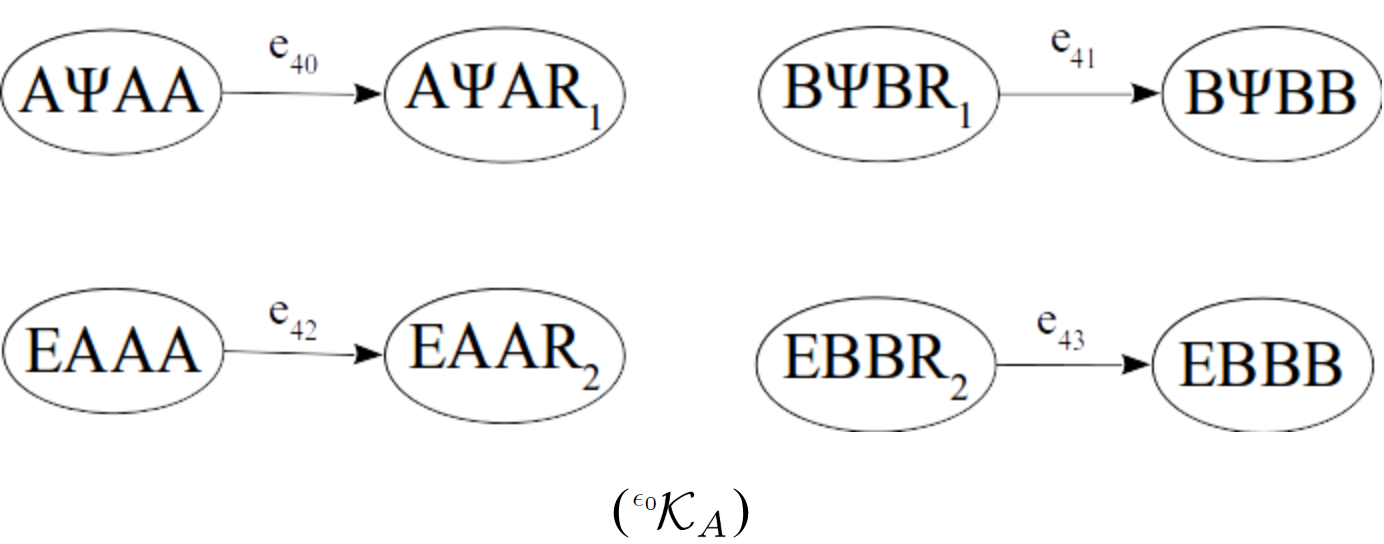}
\caption{State transition graphs of inter-agents capabilities.}
\label{fig:interagents_caps}
\end{figure}

The optimal $\mathcal{T}$ computed by SPECTER with the heuristic option after the failure mode incorporation resulting to a composition of $6$ open chain modules and it is presented in Fig. \ref{fig:module_chain1}. The solution utilized the heuristic option. In words, ($\Tau_1$) $R_1$ navigates from $\text{E}$ to the pick-up location $\text{A}$ in $10s$, ($\Tau_2$) $W_1$ moves from $\Gamma$ to $\text{A}$ to load the payload on $R_1$, ($\Tau_3$) $W_1$ loads $I_1$ on $R_1$ in $3s$, ($\Tau_4$), $A_1$ moves from $\text{A}$ to $\text{B}$ carrying the payload in $15s$, ($\Tau_5$), $W_1$  moves to $\text{B}$, and ($\Tau_6$) $W_1$ unloads $I_1$ at $\text{B}$. The task is completed in $55s$. Even though $R_2$ is closer to $I_1$, it is not utilized due to its failure mode (Fig. \ref{fig:agent_failure}). 

As defined in section \ref{sec:MC_PF}, $q(\Tau_0^{-1}) = x_{0,\mathcal{S}} = \text{E}\Psi \Gamma \text{A}$ and $p(\Tau_0^{-1}) = x_d = \text{E}\Psi \Gamma \text{B}$ (line 2, Algorithm \ref{algo:funheuristic}). From the solution $\mathcal{P}$ provided in step 4 of Algorithm \ref{algo:funheuristic} the first time that the expression in step 10 holds is for $k=7$ where Algorithm \ref{algo:funheuristic} returns the solution to Algorithm \ref{alg:algo-modulechain} and the module chain is constructed. 

The pre-processing runtime to construct agents and environment $\epsilon_0$-NFAs is $3.816 \times 10^{-2}$. The problem solving runtime to find the optimal $\mathcal{T}$ with the heuristic option is $1.245 \times 10^{-3}$ seconds whereas $5.952 \times 10^{-3}$ seconds running the complete option. SPECTER yields the same solution for the optimal $\mathcal{T}$ with the complete option for all the $80$ possible $x_d$'s for which we have that $proj(x_d, b)=\text{B}$. Both module chains $\mathcal{T}$ provided either by the ``Complete'' or ``Heuristic'' options are identical in this scenario case. 
\begin{figure}[ht]
\centering
\includegraphics[scale=0.25]{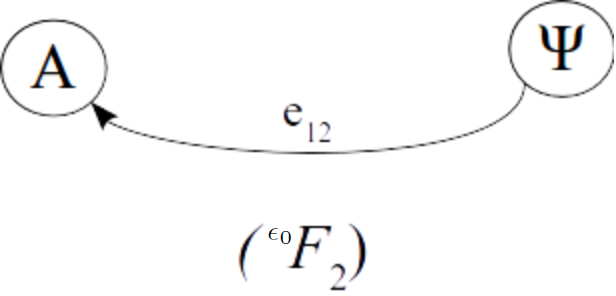}
\caption{State transition graphs of failure mode.}
\label{fig:agent_failure}
\end{figure}
\begin{figure}[ht]
\centering
\includegraphics[scale=0.175]{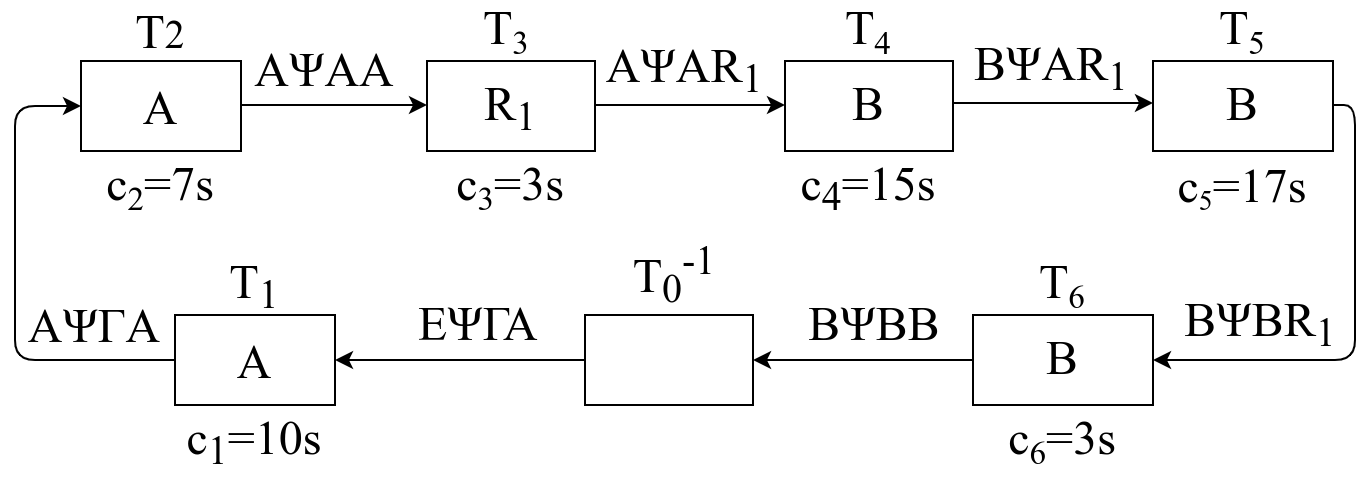}
\caption{Task plan as the optimal $\mathcal{T}$ with the ``Heuristic'' option.}
\label{fig:module_chain1}
\end{figure}

\subsection{Case Study II}
In this case study, we identified a pattern of two regular workflows concerning the manufacturing of semi-finished products 1 and 2 and final products 1. The construction of semi-finished products 1 require raw materials 1 and 2 whereas the semi-finished products 2 require raw material 3. The semi-finished products 1 are produced by the injection machine 1. The semi-finished products 2 are created by injection machine 2. Entities produced by injection machines are collected at stations $J$ or $C$. The semi-finished products 1 are passed through the conveyor 2 to the packing machine. The final products 1 are created from semi-finished product 1, which are passed through the conveyor 2 to the packing machine. Packed entities are collected at station $F$ and then are transported to Warehouse.

To model the two workflows requires to define the agents involved in the workflow processes, such as material, products, robots and workers. We consider that $n= 9$ with ${}^{\epsilon_0} \!A_1, {}^{\epsilon_0} \!A_2, {}^{\epsilon_0} \!A_3$ define the raw materials 1, 2 and 3, ${}^{\epsilon_0} \!A_4, {}^{\epsilon_0} \!A_5$ denote the semi-finished products 1 and 2, ${}^{\epsilon_0} \!A_6$ denotes the final product 1, ${}^{\epsilon_0} \!A_7$, ${}^{\epsilon_0} \!A_8$ model robots 1 and 2 and ${}^{\epsilon_0} \!A_9$ models the human-worker, where $\forall i \in \{1,...n\}: {}^{\epsilon_0} \!A_i \in {}^{\epsilon_0} \!\mathcal{A}$. The initial states of the agents are: Raw materials 1, 2, 3 are stored, each robot is at its docking station and the human is in Warehouse. The objective tasks are: (a) produce ${}^{\epsilon_0} \!A_6$ and store it in Warehouse and (b) produce ${}^{\epsilon_0} \!A_4$ and ${}^{\epsilon_0} \!A_5$, both in minimum time. Robot coordination and navigation is handled by appropriate controllers whose description is beyond the scope of the current work.

\begin{figure}[ht]
\centering
\includegraphics[scale=0.083]{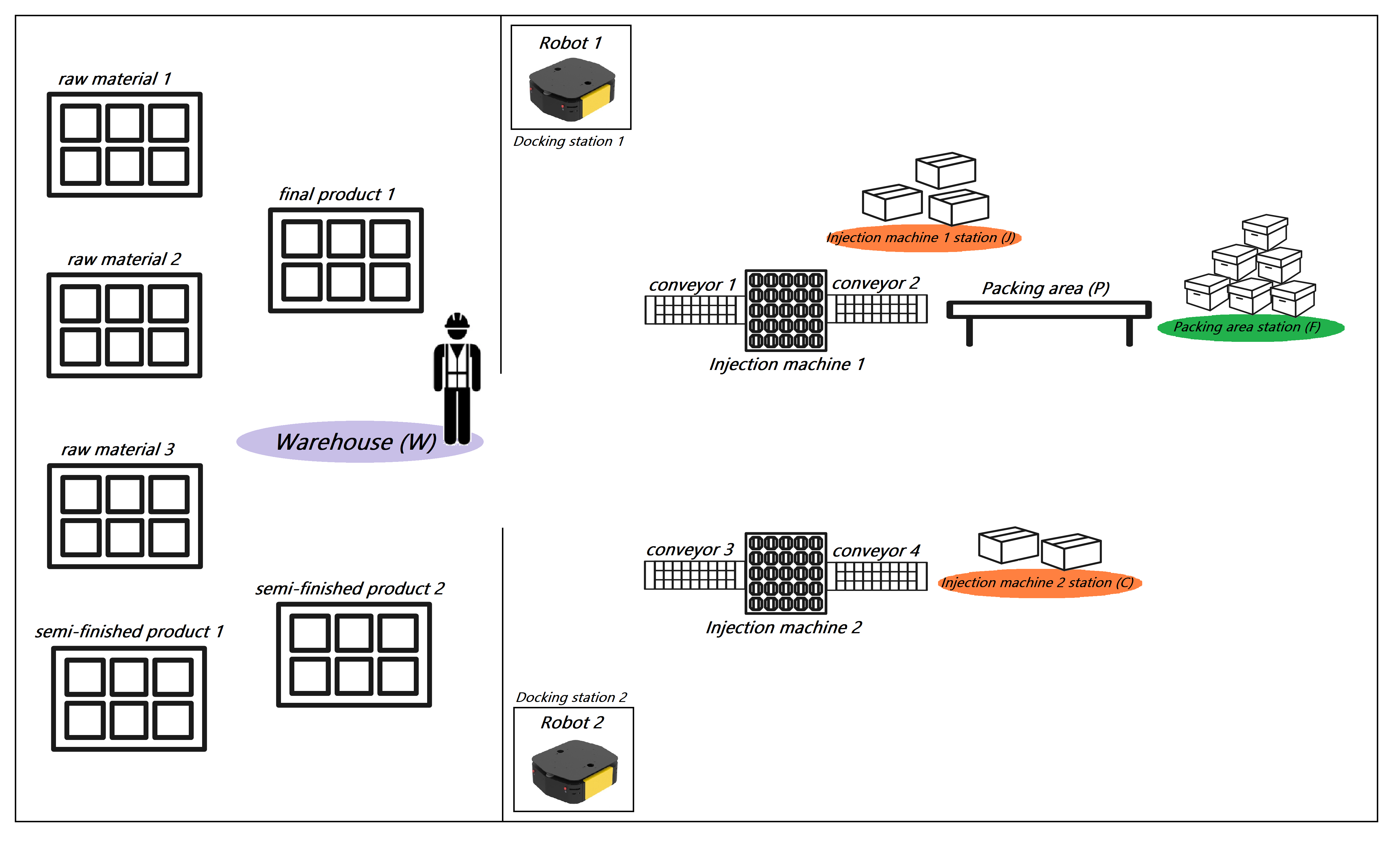}
\caption{Factory plant.}
\label{fig:case2}
\end{figure}

Models of ${}^{\epsilon_0} \!A_7$, ${}^{\epsilon_0} \!A_8$ and ${}^{\epsilon_0} \!A_9$ are constructed by combing the agents capabilities and constraints. Capabilities and constraints of ${}^{\epsilon_0} \!A_1, {}^{\epsilon_0} \!A_2, {}^{\epsilon_0} \!A_3, {}^{\epsilon_0} \!A_4, {}^{\epsilon_0} \!A_5, {}^{\epsilon_0} \!A_6$ are defined as in inter-agent capabilities and constraints. ${}^{\epsilon_0} \!\mathcal{S}$ is constructed, where $\footnotesize{X_\mathcal{S} = \prod_{i=1}^{n} |X_{A_i}| = 1.875 \times 10^6}$ states. We state $\gamma$ as $proj(x_d, b)=JCW$, where $b=000111000$. 

The optimal $\mathcal{T}$ computed by SPECTER with the heuristic option resulting to a composition of $24$ open chain modules. In words, ($\Tau_1$) raw materials 1, 2, 3 are stored in $W$, robots are at their docking stations, worker is at $W$, ($\Tau_2$) robot 1 goes to $W$, ($\Tau_3$) worker loads the raw material 3 to robot 1, ($\Tau_4$) worker loads raw material 2 to robot 1, ($\Tau_5$) worker loads raw material 1 to robot 1, ($\Tau_6$) worker goes to injection 2, ($\Tau_7$) robot 1 goes to injection 2 while carrying raw materials, ($\Tau_8$) worker unloads raw material 3 from the robot 1 and loads the raw material 3 to the $C$, ($\Tau_9$) robot 1 goes to injection 2 while carrying raw materials 1 and 2, ($\Tau_{10}$) worker goes to injection 2, ($\Tau_{11}$) $C$ starts the production of semi-finished product 2, ($\Tau_{12}$) worker unloads raw material 2 from robots 1 and loads the raw material 2 to $J$, ($\Tau_{13}$) injection machine starts the preparation of semi-finished product 1, ($\Tau_{14}$) worker unloads raw material 1 from robot 1 and loads the raw material 1 to $J$, ($\Tau_{15}$) injection 1 starts the production of semi-finished product 1, ($\Tau_{16}$) the final product 1 is produced at $F$a, ($\Tau_{17}$) semi-product 2 is produced by $C$, ($\Tau_{18}$) worker goes to $F$a, ($\Tau_{19}$) robot 2 goes to $F$a, ($\Tau_{20}$) worker loads final product 1 on robot 2, ($\Tau_{21}$) robot 2 goes to $W$, ($\Tau_{22}$) semi-finished product 1 is produced by $J$, ($\Tau_{23}$) worker goes to $W$, ($\Tau_{24}$) worker unloads the final product 1 from robot 2 to $W$.

The pre-processing runtime to construct the agents and environment models is $3.081 \times 10^5$. The problem solving runtime to find the optimal $\mathcal{T}$ utilizing the heuristic option is $165.73$ seconds whereas $1.072 \times 10^4$ seconds required to implement the complete option. SPECTER yields the same optimal $\mathcal{T}$ using the heuristic option for all $15000$ possible $x_d$'s for which we have that $proj(x_d, b)=JCW$. $\mathcal{T}$ provided either by Algorithm \ref{alg:funcomplete} or \ref{algo:funheuristic} are identical.

\section{CONCLUSIONS}
\label{sec:Concl}
A new multi-agent task planner approach, the SuPErvisory Control Task plannER, \textit{SPECTER} has been proposed. Given the capabilities, constraints and failure modes of the agents under the framework of NFAs with $\epsilon$-transitions, SPECTER produces optimized solutions, providing the sequence of tasks for transporting the state of the environment from any initial to any given destination state. The developed algorithms can provide a complete solution with optimality guarantees whenever a solution exists. By relaxing the completeness property requirement, an option providing a significant reduction in the computational requirements is proposed, that  provides suboptimal solutions through the use of efficient  heuristics. The results of the case studies demonstrate the applicability as well as the effectiveness and validity of the proposed methodology in successfully generating optimal executions for multi-agent systems with a predetermined set of individual capabilities and constraints as well as agent coupling capabilities and restrictions. Future work will focus on combining SPECTER with supervisory control theory to enable reactive execution in dynamic environments.

\bibliographystyle{IEEEtran}
\bibliography{IEEEabrv,library2021}
\end{document}